\documentclass[11pt,a4paper]{article}
\usepackage[margin=27mm]{geometry}
\usepackage[T1]{fontenc}
\usepackage{lmodern}
\usepackage[round]{natbib}
\usepackage{amsmath,amssymb,amsthm,booktabs,graphicx,microtype}
\usepackage{tikz,pgfplots}
\usepackage[section]{placeins}
\usepackage{flafter}
\usepackage[font=small,labelfont=bf]{caption}
\usepgfplotslibrary{groupplots,fillbetween}
\pgfplotsset{compat=1.18}
\usepackage[colorlinks=true,linkcolor=blue!55!black,citecolor=blue!55!black,urlcolor=blue!55!black]{hyperref}
\newtheorem{theorem}{Theorem}
\newtheorem{proposition}{Proposition}
\newcommand{\E}{\mathbb E}

\hypersetup{pdftitle={How Many Labels Does Model Choice Need? Certificates and Budgets for Selective Prediction},pdfauthor={Tetsuji Kuboyama}}
\title{\textbf{How Many Labels Does Model Choice Need?}\\[2mm]
\large\textbf{Certificates and Budgets for Selective Prediction}}
\author{Tetsuji Kuboyama\\Computer Centre, Gakushuin University\\
\texttt{ori-cs.arxiv@tk.cc.gakushuin.ac.jp}}
\date{September 16, 2026}
\begin{document}
\maketitle
\begin{abstract}
Classifiers can make identical predictions yet require labels to compare
their selective performance: confidence ranks weight the same errors differently.
We quantify this requirement for the area under the generalized
risk--coverage curve (AUGRC). A prelabel lower bound rules out insufficient
budgets. With all labels known, a covering linear program bounds the
minimum number of labels sufficient to fix the winner (the certificate size)
within $K-1$ labels for $K$ candidates.
For fixed $K$, independent uniform orders and identical predictions, the
prelabel bound approaches one quarter of the pool. With iid Bernoulli errors
independent of the orders, every exact acquisition policy reads almost all
labels asymptotically, although a two-candidate certificate needs only half. Across 108 feature-panel comparisons on
nine datasets, disagreement labels settle every accuracy choice but no
AUGRC choice. A 20\% budget is ruled out in 96 conditions; certificates need
$56$--$57\%$ on average. On ten conditions with pretrained image classifiers, confidence-score
choice reads 68--91\% of 10,000 labels for exact selection and 50--67\%
with AUGRC tolerance $5\times10^{-4}$.
An exact stopping test works with any acquisition order. Together, these
results link confidence ranks to label budgets and certified model comparison.
\end{abstract}

\section{Introduction}
How many labels are needed to choose between trained models? For expensive
labels, this determines whether a comparison fits the budget. We fix
predictions and confidence ranks on a set of evaluation examples (the pool)
and seek the winner that all pool labels would give, under a fixed tie rule.

For accuracy, a row on which every candidate predicts the same label cannot
change their ordering. For selective prediction, where a classifier accepts
its most confident predictions first, that same row can decide the winner.
AUGRC integrates accepted-error mass over coverage
\citep{traub2024overcoming}; the earlier an error is accepted, the more it
counts. Figure~\ref{fig:overview} shows this difference using two rows.

\begin{figure}[!htbp]
\centering\includegraphics[width=\linewidth,alt={Two candidates predict zero on both rows but accept them in opposite orders. Their AUGRC values are (3y1+y2)/8 and (y1+3y2)/8. Reading y1=0 selects A for either value of y2, with ties favoring A.}]{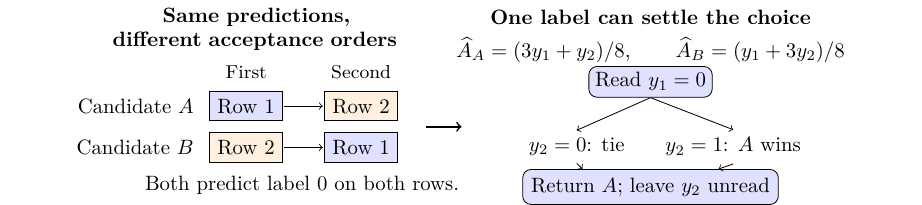}
\caption{\textbf{Same predictions do not settle a selective comparison.}
Both candidates predict 0 but accept rows in opposite orders. Before reading
labels $y_i$, either can win: $y=(1,0)$ favors $B$, and $y=(0,1)$ favors $A$.
After reading $y_1=0$, both completions select $A$ under the fixed rule that
ties favor $A$. Accuracy always ties; $\widehat A$ denotes AUGRC
(lower is better), with weights defined in Section~\ref{sec:metric}.}
\label{fig:overview}
\end{figure}

Confidence ranks can impose substantial label requirements even when
predictions agree. Here, information is measured by the number of
evaluation labels revealed.
A certificate is a label set that fixes the winner regardless of unread labels.
We give a lower bound before labeling, detect a certificate during acquisition,
and bound its minimum size within $K-1$ labels afterwards. This separates the
information a comparison requires from the cost of finding it.

\FloatBarrier
Our contributions are:
\begin{enumerate}\setlength{\itemsep}{1pt}
\item \textbf{Ranks require information.} With identical predictions and
independent uniform confidence orders, the prelabel bound tends to one quarter of the pool; with
iid errors independent of ranks, certificates need at least half, but
every exact policy reads almost all labels (Propositions~\ref{prop:random}
and~\ref{prop:acquisition}).
\item \textbf{The requirement is computable.} We give an exact stopping
test valid after adaptive queries (Theorem~\ref{thm:pool}), a covering LP
bounding the minimum certificate within $K-1$ labels
(Theorem~\ref{thm:certsize}), and a prelabel lower bound
(Theorem~\ref{thm:rankbound}).
\item \textbf{The bounds guide evaluation.} Nine datasets show when labels
on agreement rows are necessary, which budgets are insufficient, and how
requirements change with tolerance and candidate sets. Pretrained multiclass
classifiers demonstrate certified confidence-score choice.
\end{enumerate}

The paper follows the evaluation workflow: Section~\ref{sec:metric}
fixes the objective, Section~\ref{sec:fitted} gives the stopping test,
Section~\ref{sec:cert} derives label requirements and their proofs, and
Section~\ref{sec:use} turns the bounds into budget decisions.
Section~\ref{sec:experiments} combines the experimental results with their
protocols and follow-up analyses.

\section{Setting}
\label{sec:metric}
Let $X$ be an input and $Y\in\{0,1\}$ its label. Let $f$ be a classifier
and $g$ a confidence score, larger for earlier acceptance.
With loss $\ell=\mathbf1\{f(X)\ne Y\}$, let
$a_c(X)\in\{0,1\}$ indicate acceptance of the most confident fraction $c$
of the population, randomizing independently within score ties. Generalized risk and its area are
\begin{equation}
 G(c)=\E[\ell a_c(X)],\qquad A=\int_0^1G(c)\,dc.
 \label{eq:area}
\end{equation}
AUGRC integrates accepted-error mass, rather than conditional error among
accepted examples \citep{traub2024overcoming}. For binary loss, $0\le A\le1/2$.

\paragraph{Finite-pool convention.}
For $n$ fixed evaluation rows with binary losses $\ell_i$, we linearly interpolate accepted-error mass
between consecutive coverage points and average over all orders within a
confidence tie. A tied block occupying zero-based positions $[s,e)$ gives each
row the weight
\begin{equation}
 \omega_i=\frac{2n-s-e}{2n^2},\qquad
 \widehat A=\sum_i\omega_i\ell_i.\label{eq:weights}
\end{equation}
The weights sum to $1/2$. With $\bar\ell=n^{-1}\sum_i\ell_i$, the
right-endpoint sum, $\widehat A+\bar\ell/(2n)$,
can change a comparison; all results use \eqref{eq:weights}.

\paragraph{Tie weights.}
At zero-based rank $r$, an error contributes $(n-r-1/2)/n^2$ to the
trapezoidal integral. Averaging $r$ over $[s,e)$ gives \eqref{eq:weights}.
The right-endpoint sum adds $1/(2n^2)$ per error.

\section{Certifying a choice from partial labels}
\label{sec:fitted}
Freeze $K$ classifier--confidence pairs before starting label acquisition
on an evaluation pool of $n$ rows. Features,
predicted labels $\hat y_{ji}\in\{0,1\}$ and confidence ranks are known for every row;
labels are not. Querying row $i$ reveals its label $y_i$ to all candidates. The
target is the candidate with the smallest AUGRC on the full pool, with the
smaller index breaking ties: the choice one would make with all pool labels.

Scale the empirical AUGRC as $r_j(y)=2n^2\widehat A_j(y)$.
With integer weights $w_{ji}=2n^2\omega_{ji}$, this gives
\begin{align}
 r_j(y)&=c_j+\sum_i a_{ji}y_i,\label{eq:linear}\\
 c_j&=\sum_iw_{ji}\hat y_{ji},\quad a_{ji}=w_{ji}(1-2\hat y_{ji}).\notag
\end{align}

The risk is linear in the labels because each row's loss is linear in $y_i$.
For the sets of queried rows $O$ and unread rows $U$, put $b_{jki}=a_{ji}-a_{ki}$.
The exact lower and upper bounds on $r_j-r_k$ over all assignments to the unread labels are
\begin{align}
 L_{jk}&=c_j-c_k+\sum_{i\in O}b_{jki}y_i+
                  \sum_{i\in U}\min(0,b_{jki}),\label{eq:lower}\\
 H_{jk}&=c_j-c_k+\sum_{i\in O}b_{jki}y_i+
                  \sum_{i\in U}\max(0,b_{jki}).\notag
\end{align}
Each endpoint is attained by setting every unread label according to the sign
of its coefficient.

\begin{theorem}[Exact finite-pool stopping]
\label{thm:pool}
With fixed predictions and confidence weights and no constraints on the unread
binary labels, candidate $k$ is the winner under the fixed tie rule for every label
completion if and only if, for every $j\ne k$,
\begin{equation}
 L_{jk}\ge0\quad(j>k),\qquad L_{jk}>0\quad(j<k).\label{eq:stop}
\end{equation}
The condition is valid after arbitrary adaptive queries. If it fails for every
$k$, the current labels do not determine the winner.
\end{theorem}
\begin{proof}
The inequalities say that $k$ beats, or wins the index tie against, every
competitor under every completion. If one inequality fails, the completion
attaining its lower endpoint makes that competitor defeat $k$ or win the tie,
so $k$ is not a universal winner. The argument is pointwise in the observation
history, so adaptive querying does not change it.
\end{proof}
Reading row $i$ raises $L_{jk}$ by $b_{jki}y_i-\min(0,b_{jki})$, which is
either $0$ or $|b_{jki}|$. For $\tau>0$, requiring $L_{jk}\ge-\tau\cdot2n^2$
for every $j\ne k$ certifies that $k$ is within $\tau$ AUGRC of the minimizer
under every completion. Section~\ref{sec:tolerance} uses this relaxation;
$\tau=0$ retains the exact tie rule in \eqref{eq:stop}.

\begin{figure}[!htbp]
\centering\small
\setlength{\tabcolsep}{4.5pt}
\begin{tabular}{crrcccrcc}\toprule
Row & $w_{Ai}$ & $w_{Bi}$ & $\hat y_{Ai}$ & $\hat y_{Bi}$ & $b_i$ & $y_i$ & $g_i$ & Read\\\midrule
1 & 7 & 5 & 0 & 0 & $+2$ & 1 & 2 & 3rd\\
2 & 5 & 7 & 0 & 0 & $-2$ & 0 & 2 & 4th\\
3 & 3 & 1 & 1 & 0 & $-4$ & 0 & 4 & 1st\\
4 & 1 & 3 & 0 & 1 & $+4$ & 0 & 0 & 2nd\\
\bottomrule\end{tabular}\par\smallskip
$[L_{AB},H_{AB}]$: $[-6,6]\to[-2,6]\to[-2,2]\to[0,2]\to[2,2]$
\caption{\textbf{The test at work on four rows.} Confidence orders are
$1,2,3,4$ ($A$) and $2,1,4,3$ ($B$), giving
$r_A-r_B=2y_1-2y_2-4y_3+4y_4$. Write $b_i=b_{ABi}$ and $g_i=b_i y_i-\min(0,b_i)$, the increase in
$L_{AB}$ on reading row $i$. Read in decreasing $|b_i|$, breaking ties by row index;
$y_i$ and $g_i$ are shown with hindsight. Since ties favor $A$, the test
certifies $B$ only when $L_{AB}>0$, after the fourth label.}
\label{fig:worked}
\end{figure}

\paragraph{A worked example.}
Figure~\ref{fig:worked} runs the test on four rows. Row 3, read first, raises
the lower bound of $r_A-r_B$ from $-6$ to $-2$; row 4 only lowers the upper
bound. The two agreement rows then lift the lower bound to $2>0$ and certify
$B$. The certificate LP of Section~\ref{sec:cert} shows afterwards that three
labels sufficed and that row 4 was unnecessary for this labeling.

\paragraph{Why sharing the labels matters.}
An interval for each risk separately, with the unread labels chosen
independently for each candidate, is a valid but weaker test. For example, if both risks contain the same unread term $3y_i$,
that term cancels in their difference. Subtracting its separate intervals
$[0,3]$ instead gives $[-3,3]$, combining incompatible values of the same label.
Let $L^{\rm sep}_{jk}=\min_y r_j(y)-\max_y r_k(y)$ with the observed labels fixed.
\begin{proposition}[Gap from separate intervals]
\label{prop:gap}
For the risks in \eqref{eq:linear},
\begin{equation}
 L_{jk}-L^{\rm sep}_{jk}
 =\sum_{i\in U:\hat y_{ji}=\hat y_{ki}}\min(w_{ji},w_{ki}).\label{eq:gap}
\end{equation}
Hence the shared test never stops later than the separate test along the
same query order.
\end{proposition}
\begin{proof}
Write $\ell_{ji}=\mathbf1\{\hat y_{ji}\ne y_i\}$. Minimizing
$w_{ji}\ell_{ji}-w_{ki}\ell_{ki}$ with the two losses chosen independently
gives $-w_{ki}$. If predictions disagree, $y_i=\hat y_{ji}$ attains this
jointly. If they agree, their common loss takes both values 0 and 1,
whether the shared prediction is 0 or 1. The joint minimum is therefore
$\min(0,w_{ji}-w_{ki})$, whose gap from $-w_{ki}$ is $\min(w_{ji},w_{ki})$.
Observed rows contribute equally to both bounds.
\end{proof}
Only agreement rows contribute to this gap. The shared bound uses one
common loss per row, multiplied by the difference in rank weights
(Figure~\ref{fig:overview}).
Centering the risks on a common label term recovers part of the gap.
Section~\ref{sec:stopping} measures both.

\paragraph{Common-term removal.}
Subtracting $\sum_i t_i y_i$ from every risk leaves its argmin unchanged.
Using $\min(0,z)=(z-|z|)/2$, the separate comparison's remaining gap is
\begin{equation*}
 L_{jk}-L^{\rm sep,t}_{jk}
 =\frac12\sum_{i\in U}\big(|a_{ji}-t_i|+|a_{ki}-t_i|-|a_{ji}-a_{ki}|\big)\ge0.
\end{equation*}
Each term is zero exactly when $t_i$ lies between the two coefficients.
Midpoint centering therefore recovers the joint bound for two candidates;
with more candidates one center need not lie between every pair, whereas
pair-specific centering always does.

\section{How confidence ranks require labels}
\label{sec:cert}
\FloatBarrier
\subsection{Minimum certificates and their computation}
A set of observed rows and their labels is a certificate if it fixes the
winner under every completion. Let $C(y)$ be its minimum size for complete label vector $y$
and winner $k$, using \eqref{eq:stop}. This is the input-specific certificate complexity
\citep{buhrman2002complexity} of our decision, and every valid
acquisition path on $y$ reads at least $C(y)$ labels. For each $j\ne k$ define
\begin{align}
 D_j&=[\mathbf1\{j<k\}-L_{jk}(\varnothing)]_+,\nonumber\\
 g_{ji}&=b_{jki}y_i-\min(0,b_{jki})\ge0,\label{eq:cert-gain}
\end{align}
the deficit before any label and the amount row $i$ removes. Here
$[v]_+=\max(0,v)$ and $\varnothing$ means no observed labels; $D_j$ and
$g_{ji}$ depend on the target $k$, suppressed in their notation.
\begin{theorem}[Minimum certificate]
\label{thm:certsize}
With the integer risks in \eqref{eq:linear} and $x_i=1$ indicating that row $i$ is read,
\begin{equation}
 C(y)=\min_{x\in\{0,1\}^n}\Big\{\sum_i x_i:
                 \sum_i g_{ji}x_i\ge D_j\ \ (j\ne k)\Big\}.
 \label{eq:certsize}
\end{equation}
If $z$ is the value of its $[0,1]^n$ relaxation, rounding a basic optimal
solution upward gives a certificate of size at most $z+K-1$. The benchmark is
therefore computable within $K-1$ labels, whatever the pool size.
\end{theorem}
\begin{proof}
Revealing a set of rows adds exactly their $g_{ji}$ to each lower bound, and
\eqref{eq:stop} becomes the covering constraints, strict ties included. A
vertex of the box intersected with $K-1$ halfspaces has at most $K-1$
fractional coordinates. Rounding them up keeps feasibility, since
$g_{ji}\ge0$, and adds at most $K-1$ to the objective.
\end{proof}
In Figure~\ref{fig:worked}, any two labels contribute at most 6 toward the
deficit $D=7$, whereas rows $\{1,2,3\}$ contribute 8. The LP relaxation gives
$z=2.5$, and rounding gives a three-label certificate, so $C(y)=3$.

\paragraph{Covering two rivals with the same labels.}
The two-candidate example above has one deficit to cover. With three
candidates, a useful label for one comparison may do nothing for another.
Figure~\ref{fig:certificate-lp} gives a separate four-row example with
$y=(1,0,0,1)$ and ties resolved in the order $A,B,C$. Its full-label
integer risks are $(r_A,r_B,r_C)=(8,15,13)$, so the winner is $A$.
Before any label, the lower bounds against $A$ are $(-13,-15)$.
Reading rows 2 and 3 supplies gains $(16,12)$: enough against $B$, but
not against $C$. Rows 1 and 2 supply $(16,20)$ and cover both deficits.

\begin{figure}[!htbp]
\centering\includegraphics[width=\linewidth,alt={Three candidates on four rows. For certifying A, the gains are (4,12,4,0) against B and (8,12,0,0) against C. The fractional LP solution (3/8,1,0,0) has value 11/8; rounding selects rows 1 and 2.}]{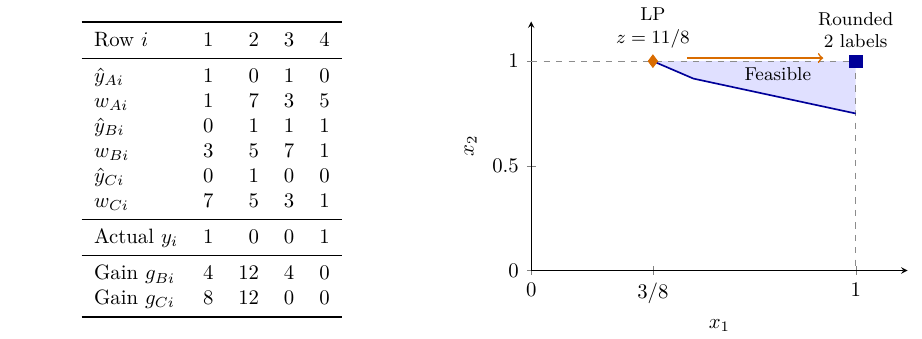}
\caption{\textbf{One label set must cover every rival.}
Left: fixed predictions and rank weights, followed by the actual labels
and their gains for certifying $A$. Each weight vector permutes
$(7,5,3,1)$. Right: the LP feasible region in the slice $x_3=x_4=0$.
Rounding $(3/8,1,0,0)$ upward gives the two-label certificate $\{1,2\}$.
The labels are shown for the hindsight calculation; an acquisition policy
does not know them before querying.}
\label{fig:certificate-lp}
\end{figure}

The LP solution $(3/8,1,0,0)$ has value $11/8$. To see that no fractional
solution is better, divide the constraint for $C$ by 8 and use $x_2\le1$:
\[
 \sum_{i=1}^4x_i\ \ge\ x_1+x_2
 \ \ge\ \frac{15}{8}-\frac{x_2}{2}\ \ge\ \frac{11}{8}.
\]
This bound holds for all four variables, not only the plotted slice.
Thus $\lceil11/8\rceil\le C(y)\le2$, which determines $C(y)=2$.
Static range reads rows in the order $2,3,1,4$ and stops after three
labels. The certificate identifies what was sufficient; the extra query
reflects the cost of finding that set without knowing the labels.

\paragraph{Verifying certificate bounds.}
For $\lambda\ge0$, weak duality gives the LP lower bound
$\lambda^TD+\sum_i\min(0,1-\lambda^Tg_i)$, where $g_i=(g_{ji})_{j\ne k}$.
We evaluate its ceiling using rational arithmetic and check rounded subsets
by integer covering sums.

\paragraph{Two candidates.}
For $K=2$, there is one covering constraint. Sorting the realized gains
$g_{ji}$ and taking the shortest prefix that meets $D_j$ gives an exact
minimum certificate. Replacing a selected row by an unselected row with
larger gain cannot break the constraint. For more candidates, each row
contributes to several constraints, and their joint coverage matters.

\FloatBarrier
\subsection{A bound before labels are read}
\label{sec:prelabel}

Part of the requirement is visible before any label is read, because
$g_{ji}\in\{0,|b_{jki}|\}$ whichever value $y_i$ takes.
For each possible target $k$, let $m_{jk}$ be the smallest $m$ such that the $m$ largest values of
$|b_{jki}|$ sum to at least $D_j$ ($m_{jk}=\infty$ if none does).
Sorting these potential gains gives a label-free bound for each pair.

\begin{theorem}[Prelabel lower bound]
\label{thm:rankbound}
For every label vector $y$ with winner $k^*$,
\begin{equation}
 C(y)\ \ge\ \max_{j\ne k^*}m_{jk^*}\ \ge\ \underline C:=\min_k\max_{j\ne k}m_{jk},
 \label{eq:rankbound}
\end{equation}
and $\underline C$ depends only on predictions and confidence ranks.
\end{theorem}

\begin{proof}
A certificate for $k$ must contain rows whose $|b_{jki}|$ sum to at least
$D_j$ for every $j$, which gives both inequalities. All coefficients and
deficits are known before labels. A lower-index duplicate of $k$ gives
$m_{jk}=\infty$, matching the tie rule.
\end{proof}

In Figure~\ref{fig:worked}, the largest potential gains are 4 and 4,
so either target needs at least two rows even before labels are known:
\begin{center}\small
\begin{tabular}{ccc}
Before labels & With all labels & Static order\\
$\underline C=2$ & $C(y)=3$ & 4 labels read
\end{tabular}
\end{center}
A one-label budget is ruled out in advance. The actual labeling needs
three labels, while this order reads one extra.

\paragraph{A worst-case rank example.}
For $n\ge2$, let both candidates predict zero, let $A$ rank $2,\ldots,n,1$
and $B$ rank $1,\ldots,n$, and prefer $A$ at ties. Their integer contrast is
$r_B-r_A=2((n-1)y_1-\sum_{i=2}^n y_i)$. At $y=0$, observing rows $2,\ldots,n$
suffices, and each is necessary: an unread such row can be one while $y_1=0$,
making $B$ win. Hence $C(0)=n-1$, whereas accuracy ties for every completion
and needs zero labels. This shows what confidence ranks alone can require,
not its typical frequency.

\FloatBarrier
\subsection{Independent confidence orders}
\label{sec:random}
\begin{proposition}[Random orders]
\label{prop:random}
Let $K$ candidates predict identically and let their confidence orders be
independent uniformly random permutations without ties. Then
$\underline C/n\to1/4$ in probability for every fixed $K\ge2$. If the
losses are iid Bernoulli$(q)$, $q\in(0,1)$, independent of the orders, then
$\Pr(C(y)<(\tfrac12-\varepsilon)n)\to0$ for every $\varepsilon>0$, and for
$K=2$, $C(y)/n\to1/2$ in probability.
\end{proposition}
\begin{figure}[!htbp]
\centering\includegraphics[width=\linewidth,alt={Analytic rank geometry for independent uniform confidence orders. The region with the largest rank differences contains one quarter of the rows and supplies half of the total difference mass.}]{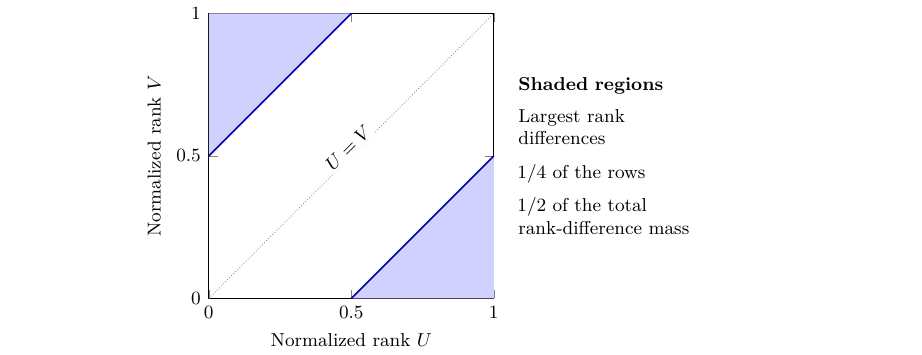}
\caption{\textbf{Where the quarter comes from.} Analytic limiting geometry
for two candidates with identical predictions and independent uniform
orders. Normalized ranks $U,V$ fill this square uniformly. The shaded
quarter carries half of the total absolute rank difference, the limiting
deficit to be covered. These are potential gains; actual gains depend on labels.}
\label{fig:rankgeometry}
\end{figure}

Figure~\ref{fig:rankgeometry} explains the first limit. For independent
uniform normalized ranks $U,V$, the rows with $|U-V|>1/2$ have probability
$1/4$ but carry half of $\E|U-V|=1/3$. Thus even the largest potential
gains need a quarter of the rows to cover the deficit. Once labels are
fixed, each gain is either zero or its potential value. Under the
independent-error model, about half of the rows provide useful gains;
their total is close to the deficit, so a two-candidate certificate
needs almost all of them. The proof below establishes these limits.
Accuracy needs zero labels to settle the tie in the same model.

\begin{proof}[Proof of Proposition~\ref{prop:random}]
\emph{Setup.} With identical predictions and no ties, the zero-based ranks
$\rho_A,\rho_B$ give $d_i:=w_{Ai}-w_{Bi}=2(\rho_B(i)-\rho_A(i))$, and the integer risks
satisfy $r_A-r_B=\sum_id_i\ell_i=:M$, where $\ell_i$ is the loss of row $i$.
Since both rank vectors are permutations of $0,\ldots,n-1$, $\sum_id_i=0$;
with $T=\sum_i|d_i|$, the positive and the negative parts of $d$ each total
$T/2$. Index rows by $i=0,\ldots,n-1$ so that $\rho_A(i)=i$; then $\rho_B=\sigma$ is a uniformly
random permutation and $|d_i|/2n=|\sigma(i)-i|/n$.

\emph{Lemma (permutation statistics).} For bounded Riemann-integrable
$f:[0,1]^2\to\mathbb R$, $\frac1n\sum_if(i/n,\sigma(i)/n)\to\iint f$ in
probability. Indeed the mean is the Riemann sum
$n^{-2}\sum_{i,j}f(i/n,j/n)\to\iint f$, and by Hoeffding's variance formula
for $\sum_ic_{i\sigma(i)}$ with $c_{ij}=f(i/n,j/n)$
\citep{hoeffding1951combinatorial},
\begin{align*}
 \operatorname{Var}\Big[\sum_ic_{i\sigma(i)}\Big]
 &=\frac1{n-1}\sum_{i,j}\big(c_{ij}-\bar c_{i\cdot}-\bar c_{\cdot j}+\bar c_{\cdot\cdot}\big)^2\\
 &\le\frac{16n^2\|f\|_\infty^2}{n-1},
\end{align*}
so the variance of the average is $O(1/n)$ and Chebyshev applies. With
$f=\mathbf1\{|x-y|>s\}$ and $f=|x-y|\mathbf1\{|x-y|>s\}$ this gives, for
every fixed $s\in[0,1]$,
\begin{align*}
 F_n(s)&:=\tfrac1n\#\{i:|d_i|>2ns\}\to(1-s)^2,\\
 G_n(s)&:=\tfrac1{2n^2}\textstyle\sum_{|d_i|>2ns}|d_i|\to H(s):=\int_s^1 2t(1-t)\,dt,
\end{align*}
since $|U-V|$ has density $2(1-t)$ for independent uniforms $U,V$; explicitly
$H(s)=\tfrac13-s^2+\tfrac23s^3$.

\emph{Rank-only bound.} The deficits are $T/2$ for $A$ and $T/2+1$ for $B$,
so $\underline C$ is the number of largest $|d_i|$ whose sum first reaches
$T/2$. Fix $\varepsilon>0$. $H$ is strictly decreasing with $H(0)=1/3$ and
$H(1/2)=1/6$, so $H(\tfrac12+\varepsilon)<\tfrac12H(0)<H(\tfrac12-\varepsilon)$.
By the lemma, with probability tending to one
$G_n(\tfrac12+\varepsilon)<\tfrac12G_n(0)=T/4n^2<G_n(\tfrac12-\varepsilon)$:
the rows with $|d_i|>2n(\tfrac12+\varepsilon)$ do not reach $T/2$ and those
with $|d_i|>2n(\tfrac12-\varepsilon)$ do. Hence
$F_n(\tfrac12+\varepsilon)\le\underline C/n\le F_n(\tfrac12-\varepsilon)+1/n$,
and both ends tend to $(\tfrac12\mp\varepsilon)^2$. As $\varepsilon$ is
arbitrary, $\underline C/n\to1/4$.

\emph{Greedy optimality.} For $K=2$, \eqref{eq:certsize} has one constraint
with nonnegative coefficients and a cardinality objective; replacing any
chosen row by an unchosen row of larger gain keeps feasibility, so the
$m$ largest gains form an optimal certificate for every feasible $m$.

\emph{Certificate.} Let the losses be iid Bernoulli$(q)$, $q\in(0,1)$,
independent of the orders. Before looking at the winner, define the two row
sets $U^+=\{d_i>0,\ell_i=1\}\cup\{d_i<0,\ell_i=0\}$ and $U^-$, its
complement. If $M>0$, $B$ wins; its useful rows are $U=U^+$ with gains
$|d_i|$ and deficit $D=1+\sum_{d_i<0}|d_i|=1+T/2$, and the total gain is
$\sum_{d_i>0,\ell_i=1}d_i+\sum_{d_i<0}|d_i|-\sum_{d_i<0,\ell_i=1}|d_i|=M+T/2$.
If $M\le0$, $A$ wins, $U=U^-$, $D=T/2$ and the total gain is $T/2-M$. In
both cases the total gain exceeds the deficit by the slack
$|M|-\mathbf1\{M>0\}\le|M|$, so by greedy optimality $C(y)=|U|-t$, where
$t$ is the largest number of smallest useful gains whose sum is at most the
slack; in particular $C(y)\le|U|$.
Conditionally on the orders, the indicators $\mathbf1\{i\in U^+\}$ are
independent with success probability $q$ if $d_i>0$ and $1-q$ if $d_i<0$
(this holds for $U^+$, which does not depend on the winner; the winner-dependent
set $U$ inherits the conclusion below because it is either $U^+$ or its
complement), so
$\mathbb E[|U^+|\mid\sigma]=q\,\#\{d_i>0\}+(1-q)\,\#\{d_i<0\}$ and
$\operatorname{Var}(|U^+|\mid\sigma)\le n/4$; the lemma with $f=\mathbf1\{y>x\}$
gives $\#\{d_i>0\}/n\to1/2$ and likewise for $d_i<0$, hence
$|U^+|/n\to1/2$, $|U^-|/n=1-|U^+|/n\to1/2$, and therefore $|U|/n\to1/2$
whichever candidate wins. Also $\mathbb E[M\mid\sigma]=q\sum_id_i=0$ and
$\operatorname{Var}(M\mid\sigma)=q(1-q)\sum_id_i^2\le4q(1-q)n^3$, so
$|M|/n^{3/2}$ is bounded in probability. Fix $x\in(0,1)$ and let
$N_x=\#\{i:|d_i|\le2nx\}$, so $N_x/n\to2x-x^2\le2x$ by the lemma. If
$t>N_x$, the $t$ smallest useful gains include $t-N_x$ gains larger than
$2nx$, whence $(t-N_x)\,2nx<|M|$ and $t<N_x+|M|/(2nx)$. Therefore
\begin{equation*}
 \frac{C(y)}n\;\ge\;\frac{|U|}n-\frac{N_x}n-\frac{|M|}{2n^2x},
\end{equation*}
whose right side tends to $\tfrac12-(2x-x^2)$ in probability, while
$C(y)/n\le|U|/n\to1/2$. Letting $x\downarrow0$ proves
$C(y)/n\to1/2$.

\emph{More candidates.} For $K$ candidates, every pair $(j,k)$ has
$w_j-w_k$ distributed as $d$ above, so $m_{jk}/n\to1/4$ in probability for
each of the finitely many pairs, and $\underline C=\min_k\max_{j\ne k}m_{jk}$
gives $\underline C/n\to1/4$. For the certificate, let $k^*$ be the winner.
Any certificate for $k^*$ satisfies the constraint of each pair $(j,k^*)$ in
\eqref{eq:certsize}, and $k^*$ also wins the two-candidate comparison
against $j$, so $C(y)\ge C_{jk^*}(y)$, the two-candidate certificate size of
that pair. Each $C_{jk^*}(y)/n\to1/2$ by the argument above (the case
$M\le0$ covers the index tie rule in either direction), and a union bound
over the pairs gives $\Pr(C(y)<(\tfrac12-\varepsilon)n)\to0$.

\end{proof}

\FloatBarrier
\subsection{The cost of finding a certificate}
\label{sec:acquisition-limit}

\begin{proposition}[Acquisition versus certificates]
\label{prop:acquisition}
Under the independent-order and independent-error assumptions of
Proposition~\ref{prop:random}, fix the common predictions in advance. Let
$N_\pi$ count distinct labels read by any adaptive policy that uses only these
predictions, ranks, observed losses and independent randomness,
and stops only when it certifies the exact winner for every completion. For fixed $K\ge2$,
$N_\pi/n\to1$ in probability, uniformly over such policies. For $K=2$,
$(N_\pi-C(y))/n\to1/2$ in probability.
\end{proposition}
Even an optimal policy must find the useful labels without knowing them.
Thus the gap to a hindsight certificate need not be removable by better
query ordering.

Figure~\ref{fig:three-limits} puts the two propositions on one scale.
The prelabel bound uses the largest possible gains; the certificate uses
the gains that actually occur. Acquisition must find useful labels
without seeing the unread losses. These are three different questions,
even for the same predictions and ranks.

\begin{figure}[!htbp]
\centering\includegraphics[width=\linewidth,alt={Three asymptotic label fractions: the prelabel lower bound is one quarter, the minimum two-candidate certificate is one half, and every exact acquisition policy reads a fraction tending to one under the stated independence assumptions.}]{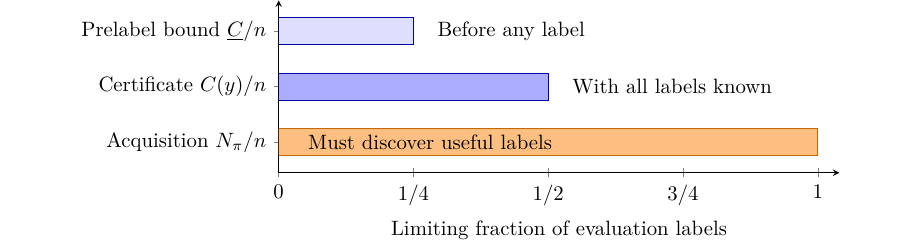}
\caption{\textbf{Enough labels, and the cost of finding them.}
For two candidates with identical predictions, independent uniform
confidence orders and iid Bernoulli$(q)$ losses independent of those
orders, $0<q<1$, the three fractions converge in probability to
$1/4$, $1/2$ and $1$. These are asymptotic limits, not measured savings.
The last limit holds uniformly over the policies in
Proposition~\ref{prop:acquisition}; better query ordering cannot remove
the entire gap to the hindsight certificate in this model.}
\label{fig:three-limits}
\end{figure}
\FloatBarrier

\begin{proof}[Proof of Proposition~\ref{prop:acquisition}]
Fix $q\in(0,1)$ and the common predictions before drawing ranks and losses.
First take two candidates and use $d_i=w_{Ai}-w_{Bi}$ from the preceding
proof. Fix the ranks and the policy's random seed, independent of the losses.
Extend its queries after stopping, if needed, until every row has been read.
If $I_t$ is the row chosen at time $t$, then
\[
 S_t=\sum_{s=1}^t d_{I_s}(\ell_{I_s}-q)
\]
is a martingale: $I_t$ depends only on ranks and previously read losses,
so the next loss is still Bernoulli$(q)$. Its conditional second moment is
$\mathbb E[S_n^2\mid\text{ranks}]=q(1-q)\sum_i d_i^2\le4q(1-q)n^3$.
Consequently, for any fixed $a>0$, the maximal inequality gives
\[
 \Pr\{\max_{t\le n}|S_t|\ge an^2\mid\text{ranks}\}
 \le \frac{4q(1-q)}{a^2n}.
\]
This bound is uniform over policies and their random seeds.

For unread rows $U$, write $W_+=\sum_{i\in U,d_i>0}d_i$ and
$W_-=\sum_{i\in U,d_i<0}|d_i|$. Since $\sum_i d_i=0$, the exact completion
interval for $r_A-r_B$ after $t$ queries is
\[
 [S_t-qW_+-(1-q)W_-,\quad S_t+(1-q)W_++qW_-].
\]
Let $p=\min(q,1-q)>0$. If $|S_t|<p(W_++W_-)$, the interval contains
both negative and positive values, so neither candidate can be certified,
regardless of the tie rule. Thus any valid stopping time requires
$|S_t|\ge p\sum_{i\in U}|d_i|$.

Fix $\varepsilon>0$ and choose $x>0$ with $2x-x^2<\varepsilon/2$.
The permutation lemma in the preceding proof implies, with probability
tending to one, that fewer than $\varepsilon n/2$ rows have
$|d_i|\le2nx$. Every unread set of at least $\varepsilon n$ rows therefore
has $\sum_{i\in U}|d_i|\ge\varepsilon x n^2$. The maximal bound with
$a=p\varepsilon x$ shows that the probability of a valid stop with this
many unread rows tends to zero, uniformly over policies. Hence $N_\pi/n\to1$.

For fixed $K>2$, apply the same argument to every candidate pair along
the policy's queries. A union bound over the finitely many pairs ensures
that, while $\varepsilon n$ rows remain, no pairwise winner is certified.
Certifying the overall winner requires certifying it against each rival,
so again $N_\pi/n\to1$. For $K=2$, Proposition~\ref{prop:random} gives
$C(y)/n\to1/2$, so subtraction gives the stated gap.
\end{proof}

\FloatBarrier
\subsection{Tolerance, multiclass predictions and other objectives}
\label{sec:extensions}
\paragraph{Certificates with a risk tolerance.}
For $\tau>0$, $C_\tau(y)$ is the smallest label set certifying some candidate
within $\tau$ AUGRC of the minimum under every completion: replace $D_j$ by
$[-\lfloor2n^2\tau\rfloor-L_{jk}(\varnothing)]_+$ and minimize over the
returned candidate; the $K-1$ bound carries over. At $\tau=0$, set
$C_0(y)=C(y)$ and use the original deficits in \eqref{eq:cert-gain},
including the index tie rule.

\paragraph{Fixed multiclass predictions.}
The finite-pool bounds and certificates also apply when candidates share a multiclass
prediction $\hat y_i$. Set $e_i=\mathbf1\{\hat y_i\ne y_i\}$; then
$r_j=\sum_iw_{ji}e_i$, so use $c_j=0$, $a_{ji}=w_{ji}$ and query $e_i$.
With at least two classes and unrestricted unread labels, every binary
error completion is realizable: choose the predicted class for $e_i=0$
and any other class for $e_i=1$. Thus the bounds, certificates and stopping
test remain exact for confidence-score selection.

\paragraph{Other evaluation objectives.}
The stopping test, certificate LP and prelabel bound apply whenever each
candidate's score is a sum of binary losses with nonnegative weights fixed
before labeling. Examples include weighted error and selective error at a
fixed coverage: accepting $m\ge1$ rows chosen before labeling gives each
accepted row weight $1/m$.
Rational weights require a common integer scaling for exact ties.
The quarter and half limits use the AUGRC rank weights specifically.

\paragraph{Beyond AUGRC.}
The area under the risk--coverage curve (AURC) also weights each error by
its rank \citep[Appendix A.1.1]{traub2024overcoming}. Averaging these weights
within confidence ties gives the fixed-weight form in Section~\ref{sec:cert}. For differing multiclass predictions, a categorical formulation is needed.
Shared multiclass predictions admit the common error-bit representation
in Section~\ref{sec:cert}.

\section{Using the information bounds}
\label{sec:use}
\paragraph{Before labeling: check the budget.}
Given a budget of $B$ labels, compute $\underline C$ from the predictions and ranks.
If $B<\underline C$, no acquisition policy can certify the full-pool choice
within that budget, for any label vector. Budgets $B\ge\underline C$
remain unresolved by this necessary condition.
Computing all pair bounds by sorting costs $O(K^2n\log n)$.

If the budget is ruled out, changing the query policy alone cannot help.
Instead, increase the budget, specify an acceptable risk tolerance, or
revise the candidate set before reading labels, then recompute the bound.
This redirects effort to the evaluation design before any label is collected.

\paragraph{During labeling: choose and query a row.}
Reading row $i$ shrinks the width $H_{jk}-L_{jk}$ by $|b_{jki}|$ whatever its
label. A simple static order reads rows by decreasing
\begin{equation}
 s_i=\max_j a_{ji}-\min_j a_{ji},\label{eq:range}
\end{equation}
the largest per-row contribution to any pairwise width. The pair-sum order
instead ranks $\sum_{j<k}|a_{ji}-a_{ki}|$. Adaptive versions recompute these
scores after removing candidates that lose to another under every completion,
including the tie rule. A query can shrink a width without raising a lower
bound (Figure~\ref{fig:worked}); we evaluate these heuristics against the
minimum certificate. The variance-minimizing sampler of
\citet{hara2024active}, called VMA in the implementation of
\citet{okanovic2025selector} and used here with AUGRC-weighted binary error
and uniform labels, samples row $i$ with probability proportional to
$.01+\frac1{2n}\sum_{j<k}|a_{ji}-a_{ki}|$.

Initializing the bounds costs $O(K^2n)$; updating them after a query costs
$O(K^2)$, in addition to the cost of choosing the next row.
Figure~\ref{fig:procedure} summarizes the procedure; Theorem~\ref{thm:certsize} supplies its hindsight benchmark.

\begin{figure}[!htbp]
\centering
\fbox{\begin{minipage}{.94\columnwidth}\small
\textbf{Input:} predictions $\hat y_{ji}$ and confidence ranks of $K$
frozen candidates on $n$ unlabeled rows; a query policy.\par
\textbf{Output:} the full-pool AUGRC winner, and the labels read.
\begin{enumerate}\setlength{\itemsep}{1pt}
\item Compute $c_j$, $a_{ji}$ from \eqref{eq:linear} and every $L_{jk}$ from
\eqref{eq:lower} with all rows unread.
\item If some $k$ satisfies \eqref{eq:stop}, return $k$ and the labels read.
\item Otherwise choose an unread row $i$ using the policy, read $y_i$, and
add $b_{jki}y_i-\min(0,b_{jki})$ to every $L_{jk}$. Go to 2.
\end{enumerate}
\textbf{Afterwards} (benchmark, needs all labels): bound $C(y)$ by
LP relaxation and rounding in Theorem~\ref{thm:certsize}; compare with the labels read.
\end{minipage}}
\caption{\textbf{Certified comparison in three steps.} The loop never
reads a label twice and stops as soon as every unread completion agrees.
Any policy choosing unread rows is valid; Section~\ref{sec:experiments} compares several.}
\label{fig:procedure}
\end{figure}

\paragraph{When the budget is exhausted.}
Return $k_B\in\arg\min_k\delta_k(O)$, where
\[
 \delta_k(O)=\frac{\max\{0,\max_{j\ne k}[-L_{jk}(O)]\}}{2n^2}.
\]
Taking the maximum over rivals gives candidate $k$'s exact worst-case
excess AUGRC over all unread completions. Thus $k_B$ comes with a guaranteed
risk gap, computed from the current bounds without further labels.
A zero gap guarantees minimum risk; the prescribed tie-breaking winner
still requires \eqref{eq:stop}. For example, in Figure~\ref{fig:overview},
reading $y_1=1$ gives $\delta_B=0$: $B$ is optimal for either remaining
label. Yet $y_2=1$ gives a tie, whose prescribed winner is $A$.
Thus one query can guarantee minimum loss without fixing the tie-breaking
choice.

\paragraph{After labeling: measure sufficient information.}
Report lower and upper bounds on $C(y)$ alongside the number of labels read,
$N_{\rm read}$. Together they bound how many extra labels the policy read;
finding a minimum certificate
without the unread labels can itself require extra queries
\citep{deshpande2014approximation}. Proposition~\ref{prop:acquisition}
quantifies this distinction in the independent-order model. The benchmark applies to every policy
run inside the procedure of Figure~\ref{fig:procedure}.

\section{Experiments}
\label{sec:experiments}
We ask which labels selective comparisons need and how stopping rules,
query orders and tolerance affect label cost. Six prospectively fixed datasets extend
the original three. The tabular follow-ups are retrospective: matched-model,
second-generator and confidence-score protocols were fixed before their
runs, after primary benchmark results were known. The new deep-model
extension was fixed before inference or evaluation and chooses confidence
scores for shared multiclass predictions.

\FloatBarrier
\subsection{Candidates and datasets}
\label{sec:real}
Three datasets came first: Breast Cancer, Wine (class 0 versus the rest) and
Digits (odd versus even), from scikit-learn \citep{pedregosa2011scikit}. Six
UCI datasets were then fixed, with the protocol, before their tables were
downloaded: Ionosphere, Sonar, Banknote, Haberman, Spambase and MAGIC
(208--19,020 rows, 3--60 numeric features). Each dataset uses three stratified
70/30 train/evaluation splits; evaluation pools have 54--5,706
rows (details in the artifact).

We construct candidates from feature subsets (panels). A training-only
backward search removes quartile indicators while keeping the
minority-count error within one of eight budgets. This error is the sum of
minority-label counts over identical feature patterns. Each panel receives a decision tree (minimum leaf size 2)
or logistic regression (\texttt{C=1}), with maximum class probability as
confidence. Comparing fixed sets of four or eight panels gives
$9\times3\times2\times2=108$ conditions. All candidates are fixed before
evaluation labels are read.

\paragraph{Panel construction.}
Training-row quartiles define binary indicators, ordered by absolute
class-conditional mean difference. A fixed-order backward search removes
indicators while the minority-count error stays below
$E_{\rm full}+\lfloor\epsilon n_{\rm train}\rfloor$, where $n_{\rm train}$ counts training rows and $E_{\rm full}$
uses all indicators. No retained indicator can then be removed within the
budget. The eight values are $\epsilon\in\{0,.005,.01,.02,.04,.08,.12,.16\}$.
Four-candidate collections use indices $0,2,4,6$; eight-candidate collections
use all eight. The artifact provides exact orders, seeds and fitted outputs.
\FloatBarrier
\subsection{Which labels are needed}
\label{sec:needed}
\paragraph{Agreement rows carry the decision.}
We reveal the label of every row on which at least two candidates disagree
and leave the other rows unrestricted. This settles the accuracy choice in all
108 conditions. It settles the AUGRC choice in none: for every condition, two
completions of the agreement rows have different, unique AUGRC winners.
Disagreement rows are $20.91\%$ of a pool on average ($13.5$--$36.5\%$ across
datasets); after reading them,
certifying the actual winner still needs agreement-row labels amounting to
$41$--$42\%$ of the pool.

\paragraph{Budgets ruled out before labeling.}
We replay all 108 frozen conditions using only predictions and ranks to
compute $\underline C$. For every budget fraction $b$, we count conditions
with $\lfloor bn\rfloor<\underline C$ (Figure~\ref{fig:budget}). This
retrospective analysis rules out a 10\% budget in all 108 conditions and
a 20\% budget in 96/108. At 25\% the count is 34/108; at 30\% it is 11/108.
These are illustrative budget levels, not observed operational constraints.
The artifact retains every bound threshold. With $\tau=.005$, a 20\%
budget is still ruled out in 86/108 conditions. The lower bound leaves other budgets unresolved. Figure~\ref{fig:budget}
shows what static range certifies within each budget: at an 80\% budget, exact
selection certifies 30/108 choices; tolerance $.005$ certifies 76/108.

\begin{figure}[!htbp]
\centering\includegraphics[width=\linewidth,alt={Budget outcomes on nine datasets: insufficient budgets ruled out before reading labels, choices certified within budget, and unresolved comparisons. Budgets are 20, 50 and 80 percent.}]{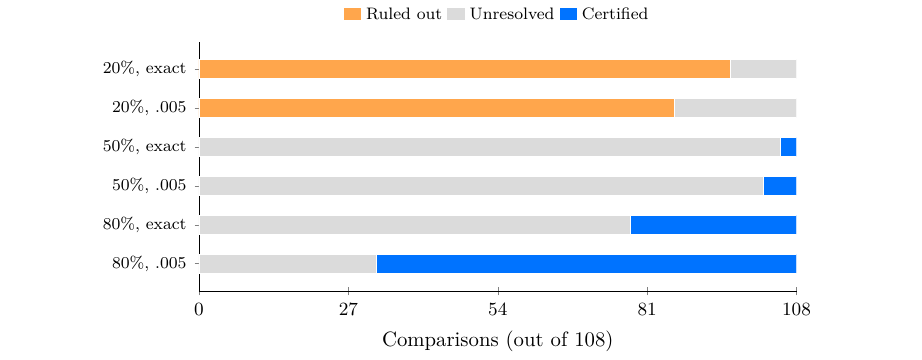}
\caption{\textbf{From a budget bound to a certified decision.}
All 108 panel comparisons, with budget and tolerance shown at left.
Orange: ruled out before labeling. Blue: certified within budget by
static range. Gray: unresolved.
At an 80\% budget, allowing $\tau=.005$ raises the certified count
from 30 to 76. These are illustrative requirements; all budget and
tolerance settings are retained in the artifact.}
\label{fig:budget}
\end{figure}

\paragraph{Minimum certificates.}
Theorem~\ref{thm:certsize} brackets $C(y)$ in all 108 conditions to within
four labels; 16 are exact. Across all nine datasets, the mean size lies between
$56.37\%$ and $57.16\%$ of the pool. On the six added datasets, certificates need
$54.61$--$55.38\%$ of the pool on average (Table~\ref{tab:extension}); on
the original three, $59.88$--$60.70\%$. The prelabel bound averages $24.50\%$ using only predictions and ranks.
At $n=5{,}706$, $K=8$, the verified LP pipeline took 22--26 ms
(Section~\ref{sec:reproducibility}).
\paragraph{Choosing confidence scores with identical predictions.}
On the same 27 outer splits, we split training rows 70/30 between classifier
and confidence training. Quartile indicators use classifier-training rows
only; all nonconstant indicators are retained. Each tree or logistic model
is fixed while comparing native confidence with negative predicted loss
\citep{franc2023optimal}. The loss regressor is a random forest with 100 trees,
depth at most 4 and minimum leaf size 20, fixed before this run without
tuning on its evaluation results. It uses the indicators and class
probability as inputs. All 54 fitted comparisons were frozen before this
run's evaluation; the pool labels had been used in earlier experiments.

Static range reads $64.46\%$ of evaluation labels, versus $85.48\%$ for random
orders (ten seeds); the prelabel bound and exact minimum certificate average
$23.38\%$ and $35.84\%$. Native confidence wins strictly in 35 comparisons,
loss regression in 17, with two ties. Means weight datasets equally.

Mean full-pool AUGRC is $.04714$ for native confidence, $.05250$ for loss
regression and $.04377$ for the selected score. Improvement over native is
nonnegative by construction on this same pool; the acquisition result is
recovering its choice with fewer labels. Labels used to fit the classifier
and loss regressor are training costs, excluded from evaluation-label savings.
For $K=2$, the minimum certificates above are computed exactly by sorting
the realized gains. Native confidence is selected in both ties.

\FloatBarrier
\subsection{Confidence choice for deep models}
\paragraph{What changes when the classifier stays fixed.}
Different confidence scores can reverse the acceptance order without
changing any predicted class. Consider two images with class-probability
vectors $(.60,.39,.01)$ and $(.55,.23,.22)$. Both predict class 1.
Maximum softmax probability ranks the first image higher ($.60>.55$),
whereas the top-two probability margin ranks the second higher
($.32>.21$). Selecting the confidence score therefore selects which
predictions to accept first. The following experiments certify that
choice using labels from the evaluation pool.

\label{sec:deep}
ResNet-20 and VGG-11-BN on CIFAR-10/100 give the original four conditions \citep{krizhevsky2009learning,chen2025cifar}. Each 10,000-image
pool compares maximum softmax probability, negative entropy and the
top-two probability margin for the same fixed predictions. The comparison certifies a score on this fixed evaluation pool.
The protocol and its relation to checkpoint development are described below.

\begin{table}[!htb]
\centering\small
\caption{\textbf{Confidence choice across five architecture families.}
Three scores share predictions on 10,000 images per condition. Numerical entries
after the model name are percentages of evaluation labels. Minimum is the offline
certificate enclosure; Read uses static range. Original and follow-up conditions
share the two CIFAR test pools.}
\label{tab:deepfamily}
\begin{tabular}{llrrrr}\toprule
CIFAR & Model & Prelabel & Minimum & Read exact & Read $\tau=.0005$\\
\midrule
\multicolumn{6}{l}{Original four conditions}\\
10 & ResNet-20 & 24.34 & 45.42--45.43 & 90.95 & 59.66\\
10 & VGG-11-BN & 20.88 & 49.73 & 88.41 & 54.04\\
100 & ResNet-20 & 15.33 & 33.85 & 68.07 & 57.20\\
100 & VGG-11-BN & 12.98 & 66.89--66.90 & 79.26 & 50.23\\
\midrule
\multicolumn{6}{l}{Six architecture-family follow-up conditions}\\
10 & MobileNetV2 x1.0 & 20.01 & 48.64 & 87.02 & 65.08\\
10 & ShuffleNetV2 x1.0 & 20.40 & 46.97--46.98 & 85.60 & 67.13\\
10 & RepVGG A0 & 19.08 & 50.26 & 86.31 & 66.64\\
100 & MobileNetV2 x1.0 & 14.79 & 64.33 & 69.85 & 57.04\\
100 & ShuffleNetV2 x1.0 & 15.30 & 60.76 & 73.13 & 60.58\\
100 & RepVGG A0 & 15.68 & 61.98 & 83.53 & 64.74\\
\bottomrule\end{tabular}\end{table}

All four exact choices match the full-pool winner, reading
68.07--90.95\% of labels (Table~\ref{tab:deepfamily}); each score wins somewhere.
At the grid's first positive tolerance, $\tau=.0005$, static range reads
50.23--59.66\%. This certifies a risk within $\tau$ of the best candidate
under every unread completion; exact selection additionally requires
identifying that winner.
Figure~\ref{fig:deepbudget} maps budgets to the smallest tested tolerance
on this frozen grid of completed runs.

As a post hoc illustration, CIFAR-100/VGG needs 3,949 and 3,109 labels
to certify the winner against its rivals separately, but 6,689--6,690
to beat both at once.

The six architecture-family follow-up conditions in Table~\ref{tab:deepfamily}
read 69.85--87.02\% for exact selection and 57.04--67.13\% at
$\tau=.0005$. The original and follow-up protocols share the two test pools.

\begin{figure}[!htbp]
\centering\includegraphics[width=\linewidth,alt={For pretrained image classifiers, the curves show the guaranteed AUGRC tolerance as a function of the evaluation-label budget. Increasing the budget tightens the guarantee.}]{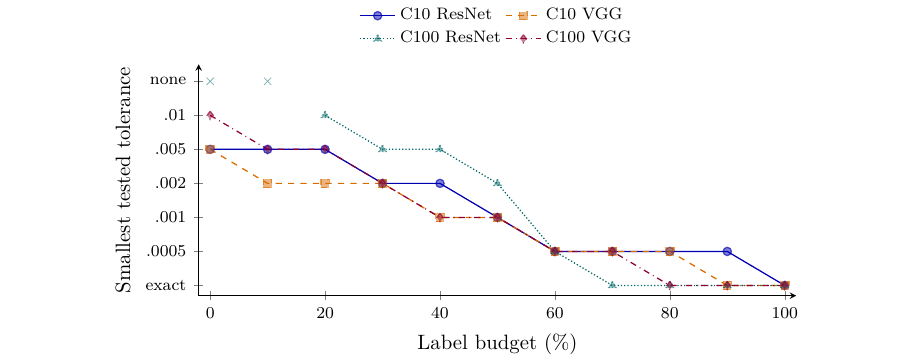}
\caption{\textbf{Choose an evaluation tolerance for a label budget.}
Smallest certified tolerance on the six-value grid, using static range
on the four original pools ($n=10{,}000$). Marks show all eleven budgets;
crosses at ``none'' mean no tested tolerance is certified. Lines guide the eye.
The tolerance axis lists grid values, not equal numerical intervals.
This is a policy-specific guarantee, not the optimal attainable tolerance.}
\label{fig:deepbudget}
\end{figure}

\subsubsection{Protocol and verification}
\paragraph{Pretrained multiclass confidence choice.}
We use the official CIFAR test splits in their original order and the public
checkpoints of \citet{chen2025cifar}, with the repository's normalization,
no augmentation and no further training. CPU inference uses float32;
softmax and scores use float64, temperature 1. The candidates, in tie
priority order, are MSP, negative predictive entropy and the top-two
probability margin. Each model supplies one fixed prediction per image.
Scores and rank weights were hashed before opening evaluation labels for
this comparison; no class-count constraints enter the certificate.
The original checkpoints used test-split validation during development.

The artifact reports all four three-score comparisons, all twelve dependent
pairwise comparisons, ten random-order seeds and all six tolerances
$0,.0005,.001,.002,.005,.01$ at eleven budgets $0,.1,\ldots,1$.
All budgets were fixed before inference.
Figure~\ref{fig:deepbudget} inverts this grid after evaluation: at each
budget it reports the smallest tested tolerance whose stopping count
fits. It does not interpolate between tolerances or optimize the policy.
For three scores, mean acquisition is 81.67\% for static range and
98.43\% for random. Pairwise minimum certificates range from 28.41\% to 49.46\%; static range
reads 50.75\%--88.84\%. Full-pool choices are recovered in every exact run.
All 176 acquisition paths and 1,666,594 prefixes were independently checked,
including integer certificate subsets and rational dual bounds. Exhaustive
three-class completion checks on 40 small pools verify the shared-error
reduction. Inference takes 5.9--9.5 seconds per model on four CPU threads;
the four three-candidate LP runs take 21--32 ms each, excluding loading.
These are single-run timings. Training costs are excluded from evaluation
savings. Public checkpoint provenance, file hashes and all outputs accompany
the code; raw images and weights are retrieved separately.

\paragraph{Architecture-family follow-up.}
After the original four outcomes, we fixed six further conditions before new
inference: MobileNetV2 x1.0, ShuffleNetV2 x1.0 and RepVGG A0 on both datasets,
from the same pinned repository. Scores, preprocessing, query orders, all
pairwise comparisons and the full tolerance/budget grid are unchanged.
Table~\ref{tab:deepfamily} reports all ten conditions; they share two test pools.
The six new exact choices match their full-pool winners. Mean static acquisition
is 80.91\% for exact selection and 63.54\% at $\tau=.0005$, compared with
81.67\% and 55.28\% in the original four. These descriptive means use dependent conditions. All 24 new primary/paired
comparisons, 264 paths and 2,504,903 prefixes pass the independent audit.
Preprocessing checks precede label access; CPU inference took 43.7--82.4 seconds
per added model. All pre-specified conditions are retained.

\FloatBarrier
\subsection{Stopping rule versus acquisition order}
\label{sec:stopping}
\begin{figure}[!htbp]
\centering\includegraphics[width=\linewidth,alt={Mean label fractions decrease from 99.43 percent for separate intervals to 90.15 for centered intervals, 85.21 for shared intervals and 80.45 for adaptive acquisition. The minimum certificate is 54.61 to 55.38 percent, leaving a gap of 25.07 to 25.84 percentage points.}]{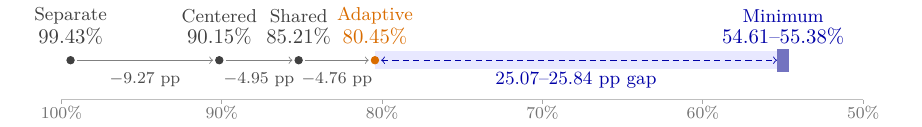}
\caption{\textbf{Where the label savings come from.} Mean label fractions
on six added datasets share a common percentage scale. The first three tests
share a static order; adaptive pair-sum changes it. Shading shows the gap to the minimum
certificate with all labels known (dark band).
Differences use unrounded means.}
\label{fig:savings}
\end{figure}
\begin{table}[!htbp]
\centering\small
\caption{Labels read (\%) by three stopping rules under the same static
order: separate intervals (S), centered separate intervals (CS) and the shared
test (J). $\underline C$ is the prelabel lower bound of
Theorem~\ref{thm:rankbound}, known before any label; $C$ brackets the minimum
certificate, computed with hindsight. Each dataset has 12 conditions.}
\label{tab:extension}
\setlength{\tabcolsep}{5pt}
\begin{tabular}{lrrrrr}\toprule
Dataset & S & CS & J & $\underline C$ & $C$\\\midrule
Ionosphere & 99.2 & 82.5 & 78.4 & 21.2 & 51.65--52.59\\
Sonar & 99.1 & 87.6 & 82.3 & 22.8 & 49.47--51.98\\
Banknote & 99.9 & 95.1 & 83.7 & 25.2 & 60.05--60.40\\
Haberman & 99.7 & 90.7 & 88.8 & 26.7 & 56.16--56.88\\
Spambase & 99.6 & 94.5 & 92.3 & 24.0 & 57.80--57.89\\
MAGIC & 99.1 & 90.5 & 85.9 & 20.0 & 52.50--52.54\\
\midrule
New 72 & 99.4 & 90.2 & 85.2 & 23.3 & 54.61--55.38\\
Original 36 & 99.4 & 92.3 & 81.7 & 26.9 & 59.88--60.70\\
\bottomrule\end{tabular}\end{table}

Figure~\ref{fig:savings} separates the two sources of savings. First we fix the static order
\eqref{eq:range} and change only the stopping rule: separate risk intervals;
separate intervals after subtracting the common label term
$\sum_it_iy_i$ with $t_i=(\min_ja_{ji}+\max_ja_{ji})/2$, which gives midpoint
centering; and the shared test \eqref{eq:stop}. Exact duplicate
candidates are merged for all three rules (Figure~\ref{fig:trace} traces
one condition). On the original 36 conditions,
the shared test saves $17.66$ points over separate intervals and $10.58$
over centering (Table~\ref{tab:extension}). On the 72 added conditions,
the three rules read $99.43\%$, $90.15\%$ and $85.21\%$ of the pool.
The shared test saves $4.95$ points over centering (median $2.46$; earlier in
$56/72$ conditions, tied in 16): a gain beyond common-term centering.

\begin{table}[!htbp]
\centering\small
\caption{Acquisition policies with the shared certificate: mean evaluation labels
read (\%). Original and added datasets are
separated. Excess brackets the mean acquisition cost above the offline
minimum certificate on the added six datasets, in percentage points.}
\label{tab:labels}
\setlength{\tabcolsep}{5pt}
\begin{tabular}{lrrr}\toprule
Acquisition & Original 3 & Added 6 & Excess (pp)\\\midrule
Random & 95.36 & 95.32 & 39.94--40.72\\
Static range & 81.71 & 85.21 & 29.82--30.60\\
Adaptive range & 81.14 & 80.75 & 25.37--26.14\\
Static pair-sum & 81.83 & 85.41 & 30.03--30.81\\
Adaptive pair-sum & 80.23 & 80.45 & 25.07--25.84\\
VMA, uniform labels & 90.64 & 90.38 & 34.99--35.77\\
Model Selector & 95.98 & 94.75 & 39.37--40.14\\
\bottomrule\end{tabular}\end{table}

Then we fix the shared test and change the order. We run seven policies:
random; static and adaptive range; static and adaptive pair-sum; uniform-label
VMA; and Model Selector's acquisition \citep{okanovic2025selector} with its
default noise parameter $.46$, each supplying queries to the same certificate test
(randomized policies use ten seeds). This comparison measures acquisition
cost under a common guarantee; it does not rank the original methods under
their different estimation or selection criteria. On the six added datasets, adaptive
pair-sum reads $80.45\%$ of the pool against $85.21\%$ (mean saving $4.76$, median $4.14$
points; fewer labels in $60/72$ conditions, equal in four, more in eight), yet
remains $25.07$--$25.84$ points above the certificate.
Across the nine datasets, a Friedman permutation test rejects equal policy
ranks: none of 99,999 within-dataset permutations reached the observed
statistic ($p_{\rm MC}=10^{-5}$ with the plus-one correction, its minimum
reportable value). Nemenyi comparisons
(Figure~\ref{fig:empirical}; \citealp{demsar2006statistical}) separate
adaptive pair-sum from random, VMA and Model Selector, but not from the other heuristics.

\begin{figure}[!htb]
\centering\includegraphics[width=\linewidth,alt={For Ionosphere with 106 rows, lower margin bounds become positive after 73 labels with the shared test, 77 with centered intervals and 106 with separate intervals. An inset shows the tie at 72 labels. The minimum certificate needs 49 to 50 labels and the prelabel bound is 16.}]{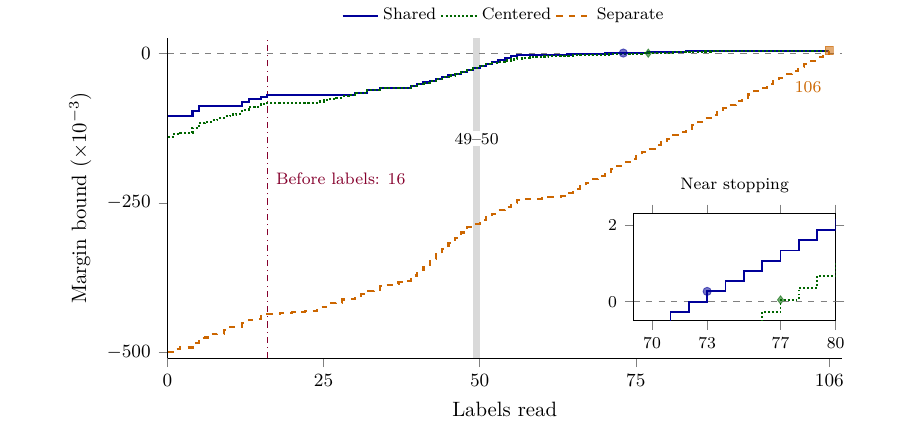}
\caption{\textbf{One condition, label by label.} Ionosphere, split 731,
trees, $K=4$, $n=106$, static range order. Curves bound
$\min_{j\ne k}(\widehat A_j-\widehat A_k)$ for the full-pool winner $k$. Dots mark the first strictly positive bounds: 73 labels (shared),
77 (centered), and 106 (separate). At 72 labels the shared bound is zero
and the tie rule prevents stopping (inset). Hindsight brackets the minimum
at 49--50 labels. The prelabel lower bound is 16: a budget of 15 cannot
certify a choice under any acquisition policy.}
\label{fig:trace}
\end{figure}

\begin{figure}[!htb]
\centering\includegraphics[width=\linewidth,alt={Critical-difference diagram of acquisition-policy ranks across nine datasets. Lower average ranks mean fewer labels read under the common exact stopping rule; connected methods are not separated by the stated post-hoc test.}]{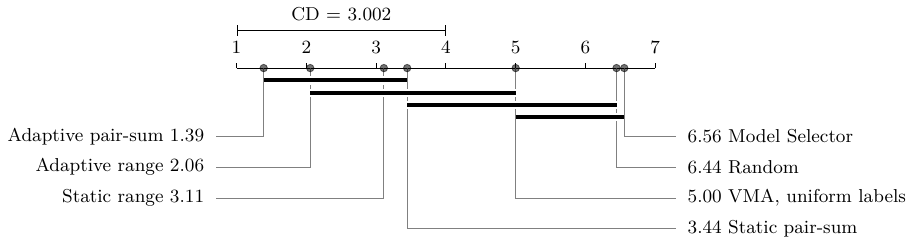}
\caption{\textbf{Acquisition ranks across nine datasets.}
Lower is better; each dataset averages its 12
conditions. Bars join policies with no detected pairwise difference (Nemenyi,
$\alpha=.05$, critical difference $=3.002$). Table~\ref{tab:labels} lists all costs.}
\label{fig:empirical}
\end{figure}

\FloatBarrier
\subsection{Tolerance and candidate-set sensitivity}
\label{sec:tolerance}
We replay the recorded static-range and adaptive pair-sum orders on all 108
conditions with $\tau\in\{0,.0005,.001,.002,.005,.01\}$ and compute $C_\tau$
for the same comparisons. The tolerance and candidate-set analyses are retrospective: their protocols
were fixed after the acquisition costs were known, and the tolerance grid
follows an exploratory calculation.

At $\tau=.005$, the minimum certificate averages $44.83$--$45.69\%$
and adaptive acquisition reads $69.74\%$, a saving of $10.64$ percentage
points. It remains $24.05$--$24.91$ points above the minimum
(Figure~\ref{fig:sensitivity}a). The returned candidate changes in $28/108$
conditions, always within tolerance. The artifact retains the full grid.

\paragraph{Candidate-set sensitivity.}
To assess candidate dependence, we use nested panel indices
$\{0,4\}$, $\{0,2,4,6\}$ and $\{0,\ldots,7\}$ from each saved collection,
without refitting. Minimum certificate sizes for exact selection average $38.42\%$,
$53.78$--$54.41\%$ and $58.95$--$59.90\%$ for $K=2,4,8$;
adaptive acquisition reads $65.57\%$, $79.52\%$ and $81.24\%$
(Figure~\ref{fig:sensitivity}b). Candidate identities and winners can change with $K$.

\paragraph{Tolerance contrasts.}
For the original 108 comparisons, moving from zero tolerance to $\tau=.005$
reduces the mean minimum certificate by $10.67$--$12.33$ percentage points.
At $\tau=.01$, the prelabel bound is $21.67\%$ and certificates need
$39$--$40\%$ on average; their responses to tolerance differ.
The budgets in Figure~\ref{fig:budget} use static range, while
Figure~\ref{fig:sensitivity} also reports adaptive acquisition.

\begin{figure}[!htb]
\centering\includegraphics[width=\linewidth,alt={Label requirements and acquisition costs versus AUGRC tolerance and nested candidate-set size. Tolerance and the set of candidates change the decision being certified.}]{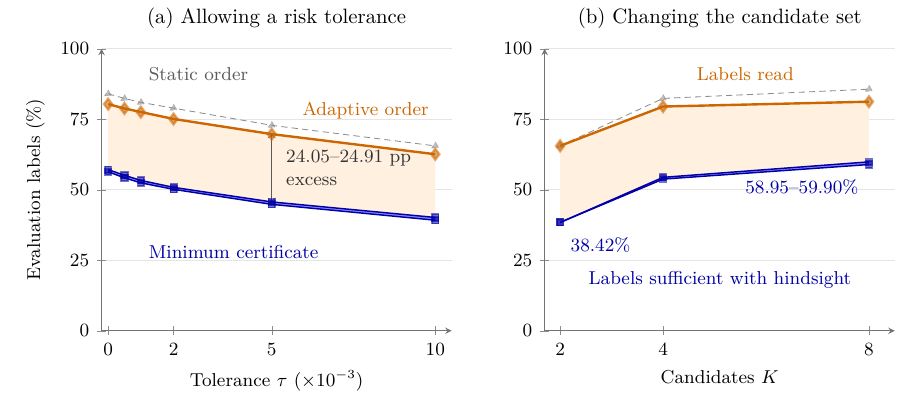}
\caption{\textbf{Decision requirements and acquisition costs change together.}
(a) Risk tolerance on all 108 primary conditions ($K=4,8$).
(b) Nested sets of 2, 4 and 8 candidates on all 54 fitted collections.
Each panel averages conditions within each of nine datasets, then datasets.
The narrow blue bands bound the mean minimum certificate. The pale area
shows the gap from its upper bound to adaptive acquisition; the annotation
gives the full excess interval. Neither shading is a confidence interval.
Points are evaluated settings, joined for readability. The static and adaptive
orders (static range and adaptive pair-sum) coincide at $K=2$.}
\label{fig:sensitivity}
\end{figure}

\FloatBarrier
\subsection{Matched pools and a second candidate generator}
\label{sec:followup}
\paragraph{Pool-size-matched comparison.}
The matched-model, second-generator and confidence-score protocols were
fixed before their respective runs, after primary results on these
benchmarks were known; all three are retrospective analyses.
For each dataset's pool size, each $K\in\{4,8\}$ and error rate
$q\in\{.1,.3\}$, 100 independent uniform-order replicates give 3,600 pools.
Predictions are identical; only $n,K$ are matched to the real collections.
Every certificate is enclosed by an integer-checked subset and a rational
weak-dual lower bound. Replicates, error rates and candidate counts receive
equal weight within each dataset. Figure~\ref{fig:followup}a compares
these means with the six fitted collections per $K$. The fitted dataset enclosures range from $48.50$--$48.84\%$ (Digits)
to $68.67$--$70.06\%$ (Wine). The matched model's
rank-only lower bound averages $25.88\%$, versus $24.50\%$ for the panels.

The 3,600 matched pools give mean certificate bounds of
$60.15$--$61.15\%$, versus $56.37$--$57.16\%$ for the fitted panels.

\paragraph{A second panel generator.}
Empirical mutual information ranks the binary
indicators using training rows only. Ties use the original indicator index.
The four panel sizes are $\lceil p/8\rceil,\lceil p/4\rceil,\lceil p/2\rceil,p$,
where $p$ is the number of nonconstant training indicators. We retain the
same splits and classifier settings and freeze all 54 fitted collections
before reading evaluation labels. Integer witnesses and rational bounds
are checked independently; predictions are refitted and query paths replayed.
The prelabel bound averages $20.42\%$, minimum certificates
$45.81$--$46.27\%$, and adaptive acquisition $70.72\%$.
Static range reads $72.19\%$ on average; adaptive pair-sum reads fewer
labels in 34 conditions, ties in nine, and reads more in 11. One collection
has duplicate loss functions; none permits a zero-label choice. Disagreement labels leave $53/54$ unresolved; adaptive acquisition remains
$24.45$--$24.91$ percentage points above the minimum certificate
(Figure~\ref{fig:followup}b).

\begin{figure}[!htb]
\centering\includegraphics[width=\linewidth,alt={Two follow-up comparisons: pool-size-matched empirical and model certificate fractions, and results from a second candidate generator based on mutual information.}]{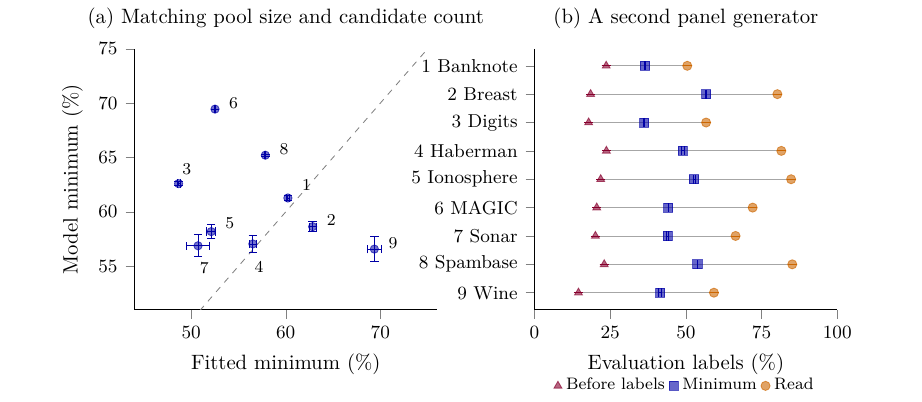}
\caption{\textbf{Two checks on the information benchmark.}
(a) Minimum-certificate fractions: fitted consistency panels versus random
orders matched by $n,K$. Each point is one dataset; bars enclose the mean
minimum, not sampling uncertainty. The diagonal is equality.
(b) The second, mutual-information panel generator: prelabel lower bound,
minimum-certificate enclosure and adaptive acquisition, averaged over six
conditions per dataset. Lines connect the same dataset; sizes are not
matched to the consistency panels.}
\label{fig:followup}
\end{figure}

\FloatBarrier
\subsection{Finite-pool simulations and correlated orders}
\label{sec:simulations}
\paragraph{Independent orders.}
The simulations use $n=50$--$5706$ for two candidates and
$n=100,500,2000$ for $K=2,4,8$. They illustrate the limits in
Proposition~\ref{prop:random}; the full grids remain in the artifact.
For two candidates, mean certificate fractions range from $.37$ at $n=50$
to $.46$ at $n=5706$; the prelabel bound rounds to $.25$ from $n=500$.
These are finite-pool observations. The certificate limit for $K>2$ is left open.

\paragraph{Correlated orders.}
Gaussian-copula orders with $\rho\in\{0,.3,.6,.9,.99\}$ gave estimated
rank-bound limits $.250,.237,.223,.209,.204$. This is a finite grid, not a
bound for the whole family: at $\rho=1$ the two orders coincide and the
requirement is zero. Real orders have mean pairwise Spearman $.51$
(dataset means $.36$--$.76$, excluding constant-confidence pairs in 14
conditions). Correlation alone does not determine the bound: a cyclic shift
by $0<s\le n/2$ rows gives $\underline C=s$, so $n=1000,s=90$ yields
Spearman $.51$ with $\underline C/n=.09$.

\section{Related work}
Selective prediction orders acceptance by confidence or conditional risk
\citep{chow1970optimum,elyaniv2010foundations,geifman2017selective,franc2023optimal}.
We use AUGRC as defined by \citet{traub2024overcoming}. Feature selection
for classification with a reject option \citep{hanczar2008classification}
and consistency-based search \citep{dash2003consistency,shin2017scwc}
address candidate construction. We fix candidates before acquiring
evaluation labels.

\citet{sawade2012active} derive sampling distributions that asymptotically
maximize the power of a risk-difference test under a fixed labeling budget.
Later label-efficient evaluation and selection
include Active Testing
\citep{kossen2021active}, loss-difference sampling \citep{hara2024active},
Model Selector \citep{okanovic2025selector}, CODA \citep{kay2025consensus},
and structured best-arm identification \citep{huang2017structured}.
We ask
how many labels suffice to fix a selective comparison under every unread
completion, and compare acquisition orders under this common guarantee.

Taking label vectors as hypotheses and full-pool winners as classes gives
equivalence-class determination \citep{golovin2010near}. Certificate complexity
is the size of the smallest sufficient label set \citep{buhrman2002complexity,grossman2020instance}. Two-candidate affine comparisons
are linear threshold evaluations, whose partial-input extrema and the
distinction between certificate cost and acquisition cost are established
\citep{deshpande2014approximation}. We quantify label requirements from confidence ranks, including the
$1/4$ prelabel limit. Under the stated model, two-candidate certificates
need asymptotically $n/2$ labels, yet every exact policy reads almost all $n$. For fixed fitted candidates, the covering formulation
bounds the minimum certificate size within $K-1$ labels.

\section{Discussion}
Confidence ranks can impose label requirements.
Across the 108 panel comparisons, a 20\% budget is ruled out in 96, and
minimum certificates average $56$--$57\%$ of the pool. Confidence-score choice
needs certificates of $35.84\%$ on average even with identical predictions.
For ten comparisons of pretrained image classifiers, exact choice reads
$68$--$91\%$ of evaluation labels; an AUGRC tolerance of $5\times10^{-4}$
reduces this to $50$--$67\%$.

The bounds rule out budgets below $\underline C$ before labeling and
benchmark acquisition cost against $C(y)$ afterwards. In the independent-order
model, even optimal acquisition reads almost all labels: a certificate is a
benchmark, not a promised query cost. The budget-to-tolerance map shows
how tolerance can reduce both costs.

\paragraph{Limitations.}
Guarantees require fixed pools, candidates and ranks, with binary labels
or shared multiclass predictions; they do not cover retraining or transfer.
Experiments use $K\le8$. Asymptotic results assume independent uniform orders;
the certificate and acquisition limits also assume iid errors independent of ranks.
No general acquisition bound for fitted data follows.

\section{Implementation and reproducibility}
\label{sec:reproducibility}
\paragraph{Computing environment.}
Model Selector retains all original candidates in its posterior update;
VMA uses the uniform-label adaptation. Experiments used an ARM64 CPU with BLAS threads set to one, Python 3.12 and SciPy
(HiGHS dual simplex for the LPs); package versions are in the artifact.
For $n=5{,}706$, $K=8$, the certificate LP pipeline, including rational
lower bounds and integer witness checks, took 22--26 ms across all six
MAGIC conditions (range of seven-run medians after one warm-up each,
excluding input loading).
The original datasets are Breast Cancer \citep{wolberg1993dataset}, Wine
\citep{aeberhard1992dataset}, and Digits \citep{alpaydin1998dataset}.
Splits are row-level; MAGIC contains simulated events. Selected indicators
can share an original variable, so their count is not a measurement cost.
Training labels are not counted as saved.

The companion source-and-evidence archive contains the frozen protocols,
outputs, independent verification code and reproduction commands. The
LaTeX source accompanying this preprint contains every plotted coordinate;
it can be rebuilt without downloading datasets or checkpoints.

\FloatBarrier
\section*{Acknowledgments}
This work was supported by JSPS KAKENHI Grant Number JP23K28151.

\section*{AI Use Statement}
I used generative AI tools to assist with mathematical analysis, experimental
work, and manuscript preparation. I set the research questions and scope,
checked the proofs, cross-checked the reported numerical results against saved
outputs using independent verification code included in the artifact, and take
responsibility for the final content.

\bibliographystyle{apalike}
\bibliography{refs}

\begin{thebibliography}{}

\bibitem[Aeberhard and Forina, 1992]{aeberhard1992dataset}
Aeberhard, S. and Forina, M. (1992).
\newblock Wine.
\newblock UCI Machine Learning Repository. Dataset.
\newblock doi:10.24432/C5PC7J.

\bibitem[Alpaydin and Kaynak, 1998]{alpaydin1998dataset}
Alpaydin, E. and Kaynak, C. (1998).
\newblock Optical recognition of handwritten digits.
\newblock UCI Machine Learning Repository. Dataset.
\newblock doi:10.24432/C50P49.

\bibitem[Buhrman and de~Wolf, 2002]{buhrman2002complexity}
Buhrman, H. and de~Wolf, R. (2002).
\newblock Complexity measures and decision tree complexity: A survey.
\newblock {\em Theoretical Computer Science}, 288(1):21--43.

\bibitem[Chen, 2025]{chen2025cifar}
Chen, Y. (2025).
\newblock {PyTorch CIFAR Models}.
\newblock \url{https://github.com/chenyaofo/pytorch-cifar-models}.
\newblock Accessed September 16, 2026; fixed checkpoint versions in the
  artifact.

\bibitem[Chow, 1970]{chow1970optimum}
Chow, C.~K. (1970).
\newblock On optimum recognition error and reject tradeoff.
\newblock {\em IEEE Transactions on Information Theory}, 16(1):41--46.

\bibitem[Dash and Liu, 2003]{dash2003consistency}
Dash, M. and Liu, H. (2003).
\newblock Consistency-based search in feature selection.
\newblock {\em Artificial Intelligence}, 151(1--2):155--176.

\bibitem[Dem\v{s}ar, 2006]{demsar2006statistical}
Dem\v{s}ar, J. (2006).
\newblock Statistical comparisons of classifiers over multiple data sets.
\newblock {\em Journal of Machine Learning Research}, 7:1--30.

\bibitem[Deshpande et~al., 2014]{deshpande2014approximation}
Deshpande, A., Hellerstein, L., and Kletenik, D. (2014).
\newblock Approximation algorithms for stochastic boolean function evaluation
  and stochastic submodular set cover.
\newblock In {\em Proceedings of the Twenty-Fifth Annual ACM-SIAM Symposium on
  Discrete Algorithms}, pages 1453--1467.

\bibitem[El-Yaniv and Wiener, 2010]{elyaniv2010foundations}
El-Yaniv, R. and Wiener, Y. (2010).
\newblock On the foundations of noise-free selective classification.
\newblock {\em Journal of Machine Learning Research}, 11(53):1605--1641.

\bibitem[Franc et~al., 2023]{franc2023optimal}
Franc, V., Prusa, D., and Voracek, V. (2023).
\newblock Optimal strategies for reject option classifiers.
\newblock {\em Journal of Machine Learning Research}, 24(11):1--49.

\bibitem[Geifman and El-Yaniv, 2017]{geifman2017selective}
Geifman, Y. and El-Yaniv, R. (2017).
\newblock Selective classification for deep neural networks.
\newblock In {\em Advances in Neural Information Processing Systems},
  volume~30, pages 4878--4887.

\bibitem[Golovin et~al., 2010]{golovin2010near}
Golovin, D., Krause, A., and Ray, D. (2010).
\newblock Near-optimal bayesian active learning with noisy observations.
\newblock In {\em Advances in Neural Information Processing Systems},
  volume~23.

\bibitem[Grossman et~al., 2020]{grossman2020instance}
Grossman, T., Komargodski, I., and Naor, M. (2020).
\newblock Instance complexity and unlabeled certificates in the decision tree
  model.
\newblock In {\em 11th Innovations in Theoretical Computer Science Conference},
  volume 151 of {\em Leibniz International Proceedings in Informatics}, pages
  56:1--56:38.

\bibitem[Hanczar and Dougherty, 2008]{hanczar2008classification}
Hanczar, B. and Dougherty, E.~R. (2008).
\newblock Classification with reject option in gene expression data.
\newblock {\em Bioinformatics}, 24(17):1889--1895.

\bibitem[Hara et~al., 2024]{hara2024active}
Hara, S., Matsuura, M., Honda, J., and Ito, S. (2024).
\newblock Active model selection: A variance minimization approach.
\newblock {\em Machine Learning}, 113:8327--8345.

\bibitem[Hoeffding, 1951]{hoeffding1951combinatorial}
Hoeffding, W. (1951).
\newblock A combinatorial central limit theorem.
\newblock {\em The Annals of Mathematical Statistics}, 22(4):558--566.

\bibitem[Huang et~al., 2017]{huang2017structured}
Huang, R., Ajallooeian, M.~M., Szepesv{\'a}ri, C., and M{\"u}ller, M. (2017).
\newblock Structured best arm identification with fixed confidence.
\newblock In {\em Proceedings of the 28th International Conference on
  Algorithmic Learning Theory}, volume~76 of {\em Proceedings of Machine
  Learning Research}, pages 593--616.

\bibitem[Kay et~al., 2025]{kay2025consensus}
Kay, J., Van~Horn, G., Maji, S., Sheldon, D., and Beery, S. (2025).
\newblock Consensus-driven active model selection.
\newblock In {\em Proceedings of the IEEE/CVF International Conference on
  Computer Vision}.

\bibitem[Kossen et~al., 2021]{kossen2021active}
Kossen, J., Farquhar, S., Gal, Y., and Rainforth, T. (2021).
\newblock Active testing: Sample-efficient model evaluation.
\newblock In {\em Proceedings of the 38th International Conference on Machine
  Learning}, volume 139 of {\em Proceedings of Machine Learning Research},
  pages 5753--5763.

\bibitem[Krizhevsky, 2009]{krizhevsky2009learning}
Krizhevsky, A. (2009).
\newblock Learning multiple layers of features from tiny images.
\newblock Technical report, University of Toronto.

\bibitem[Okanovic et~al., 2025]{okanovic2025selector}
Okanovic, P., Kirsch, A., Kasper, J., Hoefler, T., Krause, A., and G{\"u}rel,
  N.~M. (2025).
\newblock All models are wrong, some are useful: Model selection with limited
  labels.
\newblock In {\em Proceedings of the 28th International Conference on
  Artificial Intelligence and Statistics}, volume 258 of {\em Proceedings of
  Machine Learning Research}, pages 2035--2043.

\bibitem[Pedregosa et~al., 2011]{pedregosa2011scikit}
Pedregosa, F., Varoquaux, G., Gramfort, A., Michel, V., Thirion, B., Grisel,
  O., Blondel, M., Prettenhofer, P., Weiss, R., Dubourg, V., Vanderplas, J.,
  Passos, A., Cournapeau, D., Brucher, M., Perrot, M., and Duchesnay, {\'E}.
  (2011).
\newblock Scikit-learn: Machine learning in python.
\newblock {\em Journal of Machine Learning Research}, 12:2825--2830.

\bibitem[Sawade et~al., 2012]{sawade2012active}
Sawade, C., Landwehr, N., and Scheffer, T. (2012).
\newblock Active comparison of prediction models.
\newblock In {\em Advances in Neural Information Processing Systems},
  volume~25, pages 1754--1762.

\bibitem[Shin et~al., 2017]{shin2017scwc}
Shin, K., Kuboyama, T., Hashimoto, T., and Shepard, D. (2017).
\newblock {sCwc/sLcc}: Highly scalable feature selection algorithms.
\newblock {\em Information}, 8(4):159.

\bibitem[Traub et~al., 2024]{traub2024overcoming}
Traub, J., Bungert, T.~J., L{\"u}th, C.~T., Baumgartner, M., Maier-Hein, K.~H.,
  Maier-Hein, L., and J{\"a}ger, P.~F. (2024).
\newblock Overcoming common flaws in the evaluation of selective classification
  systems.
\newblock In {\em Advances in Neural Information Processing Systems},
  volume~37, pages 2323--2347.

\bibitem[Wolberg et~al., 1993]{wolberg1993dataset}
Wolberg, W., Mangasarian, O., Street, N., and Street, W. (1993).
\newblock Breast cancer wisconsin (diagnostic).
\newblock UCI Machine Learning Repository. Dataset.
\newblock doi:10.24432/C5DW2B.

\end{thebibliography}
\end{document}